%% file: main.tex
\documentclass{article} 
\usepackage{iclr2027_conference,times}

\input{math_commands.tex}

\usepackage{hyperref}
\usepackage{url}
\usepackage{enumitem}
\usepackage{xcolor}
\usepackage{graphicx}
\usepackage{booktabs}

\title{Single-Layer MeMo as a Randomized \\ Hamming-Kernel Classifier}

\author{Alessandro Straziota\\ 
Department of Enterprise Engineering\\
University of Rome ``Tor Vergata''\\
Rome, Italy \\
\texttt{alessandro.straziota@uniroma2.it}
}

\iclrfinalcopy 
\begin{document}

\maketitle

\begin{abstract}



\MeMo \citep{zanzotto2025memo} is a recent language-model architecture that stores associations between token contexts and next tokens in a correlation matrix memory.
In this work, we study its single-layer form and show that its ideal retrieval rule is a multiclass classifier based on the \emph{positional Hamming kernel}.
The \MeMo architecture represents both the sequence features and the output labels with Gaussian random codes.
Its score is therefore a doubly randomized sketch of the ideal classifier.
Under independent input and output codebooks, we bound the errors introduced by context sketching and output decoding, characterize their dependence on model and data parameters, and give a margin-based guarantee for recovering the ideal prediction.
Controlled simulations support the trends predicted by the analysis.
On a restricted \texttt{WikiText-2} next-token task, we compare single-layer \MeMo with classical baselines and show that it can offer a useful trade-off among predictive accuracy, memory, and throughput, particularly on a GPU, where its matrix operations can be parallelized.
\end{abstract}

\section{Introduction}
Language models can benefit from both general patterns and specific associations observed in data.
In transformer models, these associations are typically encoded in learned parameters, making them less directly accessible for inspection or editing.
This has motivated architectures that make memory an explicit operation \citep{Wu2022,pmlr-v267-berges25a,NEURIPS2025_a4ca07aa}.

Correlation matrix memories (CMMs) provide a classical mechanism for storing key-value associations in a distributed representation \citep{anderson1972simple,kohonen1972correlation}.  Given pairs of vectors $(k_j, v_j)$ that represents keys and values, respectively, a CMM stores the matrix $\sum_j k_jv_j^\top$.
Querying the memory with a key $q$ returns a linear combination of the stored values, where the coefficient of $v_j$ is the inner product $\ip{q}{k_j}$.
This simple algebra makes the storage rule transparent and permits local edits by adding or subtracting outer products.

\MeMo applies this mechanism to language modeling
\citep{zanzotto2025memo}.
A fixed-length sequence of tokens is encoded as a random vector and used as the key of a CMM, while the next token is represented as its value.
At retrieval time, stored next-token vectors are weighted by their similarity to the query sequence and decoded against the token codebook.
The original presentation observes that this operation behaves as an approximate counter and assigns partial credit to partially matching contexts.
However, this description leaves three questions open. 
What is the precise learning problem solved by the single-layer model?
What geometry does its notion of partial matching induce?
Finally, which roles are played by the different dimensions of the random representation?

We answer the first two questions by showing that the single-layer model is a randomized implementation of a \emph{positional Hamming-kernel classifier}.
For two contexts of length $h$, the underlying kernel is $K_H(s,q) = \frac{1}{h}\sum_{r=1}^{h}\ind\{s_r=q_r\}$, i.e., the fraction of positions at which the two contexts contain the same token.
For every possible next token $c$, the ideal classifier sums this similarity over all stored contexts whose target is $c$.
The prediction is the class with the largest accumulated similarity.
Thus, single-layer \MeMo can be understood as a \emph{kernel voting rule}, or equivalently, as a linear classifier whose class weights are unnormalized sums of training features.

\MeMo implements this classifier only approximately, using random codes for both context features and output labels.
Our main technical contribution is an analysis of this approximation: we bound the errors due to context sketching and output decoding, then combine these bounds with the ideal classifier’s margin to obtain an exact-recovery guarantee.
The analysis shows how representation dimensions, context length, and repeated labels affect recovery, thus answering the third question.

Controlled simulations test the predicted behavior.
On a restricted \texttt{WikiText-2} next-token task, we also compare \MeMo with classical baselines to assess its fidelity, predictive accuracy, memory use, and inference throughput.

\paragraph{Related work.}
As mentioned in the introduction, correlation matrix memories store key-value associations as sums of outer products \citep{anderson1972simple,kohonen1972correlation}, while \MeMo applies them to next-token prediction \citep{zanzotto2025memo}.
We study the classifier implemented by its single-layer retrieval rule.
Distributed symbolic codes \citep{plate1995holographic}, random projections \citep{johnson1984extensions}, and random features for kernel machines
\citep{rahimi2007random} provide related compression tools.
\MeMo sketches both context features and output labels.
\cite{cabannes2024scaling} study capacity and scaling laws for outer-product associative memories with independent random codes for whole inputs, whereas we analyze \MeMo's token-based context encoding and its recovery of the positional Hamming-kernel classifier.

Kernel classification provides the general framework \citep{scholkopf2002learning}.
Position-aware weighted-degree kernels have been used for sequence classification \citep{sonnenburg2005large}, and their degree-one case on fixed-length sequences is the positional Hamming kernel up to normalization.
Efficient approximations of mismatch string kernels have also been studied for sequence classification \citep{farhan2017efficient}, while recent work analyzes the expressive limits of Hamming kernels on discrete sequences \citep{amin2025biological}.
Our focus is different: we identify the count-weighted kernel vote inside \MeMo and bound when its randomized implementation preserves the ideal prediction.

\section{Preliminaries}

\label{sec:preliminaries}

\paragraph{Notation and data model.}

Let $\gV = \{1, \ldots, n\}$ be a finite vocabulary of $n$ tokens.
A context of length $h$ is an element $s = (s_1, \ldots, s_h) \in \gV^h$.
We consider a multiset $\gD = \{(s_j, y_j)\}_{j=1}^{M} \subseteq \gV^h \times \gV$, where $y_j$ is the token observed after context $s_j$.
Since $\gD$ is a multiset, repetitions contribute with their multiplicity.
Unless explicitly stated otherwise, $\gD$ is fixed and all probabilities are taken only over the random initialization of the representations.
The Euclidean inner product and norm are denoted by $\ip{\cdot}{\cdot}$ and $\norm{\cdot}$, respectively.

\paragraph{Gaussian codebooks and random projections.}
We use two token codebooks.
The input codebook $U = [u_1, \ldots, u_n] \in \R^{d \times n}$ represents tokens appearing in a context, and the output codebook $Z = [z_1, \ldots, z_n] \in \R^{d \times n}$ represents next-token labels.
In the untied model, $u_x, z_x \overset{\mathrm{i.i.d.}}{\sim} \mathcal{N}(0,I_d/d) \;\; \forall x \in \gV$, and $U$ and $Z$ are independent.
The tied construction is obtained by setting $Z = U$.
The algebraic description of \MeMo is the same in both cases, but the independence of $U$ and $Z$ will be relevant when computing expectations.

Standard Gaussian concentration implies that, for every fixed $\varepsilon \in (0,1)$, 
the events $\{ \vert \norm{u_x}^2 - 1 \vert > \varepsilon \}$ and $\{ \vert \ip{u_x}{u_{x'}} \vert >\varepsilon \}$ (with $x' \neq x$) hold with probability at most $2e^{-c d\varepsilon^2}$,
for an absolute constant $c>0$ (analogous bounds hold for $Z$).
A union bound therefore gives uniform approximate orthogonality of a codebook of size $n$
when $d \gtrsim \varepsilon^{-2} \bigl(\log n+\log(1/\delta)\bigr)$.

Let $d_h$ be the projection dimension and set $p = h d_h$.
We draw a projection $W\in\R^{d_h\times d}$ independently of the codebooks, with entries independently distributed as $\mathcal{N}(0,1/p)$.
This normalization ensures that $\E_W \norm{Wu}^2 = \frac{1}{h}\norm{u}^2$ for every fixed $u\in\R^d$. 
Equivalently, $\sqrt{h}W$ is a standard Gaussian Johnson-Lindenstrauss map into $\R^{d_h}$ \citep{johnson1984extensions}.
Thus, for a fixed codebook and distortion $\zeta\in(0,1)$, all pairwise distances among its $n$ columns are simultaneously preserved with probability at least $1-\delta$, provided that $d_h \gtrsim \zeta^{-2} \bigl(\log n+\log(1/\delta)\bigr)$.

\paragraph{The positional Hamming kernel.}
For every pair of sequences $s,q \in \gV^h$ we define the \emph{Positional Hamming Kernel} (or \emph{Positional Hamming Similarity}) as
\begin{equation}\label{eq:hamming-kernel}
    K_H(s,q) = \frac{1}{h}\sum_{r=1}^{h}\ind\{s_r=q_r\} = 1-\frac{d_H(s,q)}{h},
\end{equation}
where $d_H$ is the Hamming distance.
This is a positive-semidefinite kernel. 
To see this, let $e_x\in\R^n$ be the
one-hot vector associated with token $x$, and define $\Psi(s) = \frac{1}{\sqrt{h}} \bigl(e_{s_1}\oplus\cdots\oplus e_{s_h}\bigr) \in\R^{hn}$, where $\oplus$ indicates the concatenation operator.
Then
$\ip{\Psi(s)}{\Psi(q)} = K_H(s,q)$.
The feature map records the token appearing at each position separately.  It
does not contain joint features involving two or more positions.

\section{Problem Formulation}
\label{sec:problem}
For every candidate next token $c \in \gV$ and query context $q\in \gV^h$,  we define
the \emph{class score}
\[
S_c(q) = \sum_{j:y_j=c} K_H(q,s_j).
\]
The multiplicity of a pair in $\gD$ is retained in the sum.  We study the
following prediction problem.

\begin{problem}[Hamming-kernel next-token classification]
Given a multiset $\gD\subseteq\gV^h\times\gV$ and a query $q\in\gV^h$, compute
\begin{equation}\label{eq:ideal-prediction}
c^*(q) \in \underset{c\in\gV}{\arg\max}\, S_c(q).
\end{equation}
If the maximum is not unique, we can break the tie arbitrarily.
\end{problem}

The problem admits two complementary interpretations.  First, it is a \emph{kernel voting rule}: every stored example votes for its label with weight equal to its similarity to the query.
Second, it is a \emph{linear classification problem}.  
We have a class $c$ for every token in $\gV$, and for each class consider the unnormalized class prototype $\mu_c = \sum_{j:y_j=c}\Psi(s_j)$.
Since $\ip{\Psi(s)}{\Psi(q)} = K_H(s,q)$,
\[
    S_c(q) = \sum_{j:y_j=c} \ip{\Psi(q)}{\Psi(s_j)} =  \Psi(q)^\top \sum_{j:y_j=c} \Psi(s_j) = \ip{\Psi(q)}{\mu_c}.
\]
The prototypes are unnormalized because \MeMo is designed to accumulate occurrence counts.
Dividing $\mu_c$ by the number of examples in class $c$ would instead produce an average-similarity classifier and would remove the empirical class-frequency factor, that is a different prediction rule.

The additive structure of $K_H$ gives an alternative expression. Let $N_{c,r}(x) = \left\vert \{j:y_j=c,\; (s_j)_r=x\} \right\vert$
be the number of training examples with label $c$ and token $x$ at position
$r$.
Then we can write 
\begin{equation}\label{eq:count-table-score}
S_c(q) = \frac{1}{h}\sum_{r=1}^{h}N_{c,r}(q_r).
\end{equation}
Consequently, the ideal classifier is completely determined by the token--position--label count tables.
This identity will later provide both a geometric interpretation and a useful baseline against which the randomized \MeMo implementation can be compared.

\section{Single-Layer MeMo}
\label{sec:single-layer}

\paragraph{Sequence encoding.}
The randomized encoding of a context $s = (s_1, \ldots ,s_h)$ is the vertical concatenation 
\[
\phi(s) = \begin{bmatrix}
    Wu_{s_1}\\
    \vdots\\
    Wu_{s_h}
\end{bmatrix}
\in\R^{p},
\]
with $p = hd_h$.
The block structure preserves position: the encoding of the first token and that of the second token occupy disjoint coordinates, even though the same projection $W$ is applied to both.

\begin{lemma}
\label{prop:expected-key-kernel}
For every $s,q\in\gV^h$, 
we have $\E\ip{\phi(s)}{\phi(q)}=K_H(s,q)$.
\end{lemma}

\begin{proof}
Using the block structure, and conditioning on the input codebook $U$, $\E\ip{\phi(s)}{\phi(q)}
= \sum_{r=1}^{h} \E\ip{Wu_{s_r}}{Wu_{q_r}}
= \frac{1}{h}\sum_{r=1}^{h} \E\ip{u_{s_r}}{u_{q_r}}$.
The Gaussian codebook satisfies $\E\ip{u_x}{u_{x'}}=\ind\{x=x'\}$.
The lemma follows by substituting the sum with the definition of Hamming Kernel from \cref{eq:hamming-kernel}.
\end{proof}

Thus, the random map $\phi$ is a finite-dimensional randomized feature map for
the positional Hamming kernel.  Its approximation error is controlled jointly
by the coherence of the input codebook $U$, and by the distortion introduced by
$W$.

\paragraph{Memorization.}
Single-layer \MeMo stores every association $(s_j,y_j)$ as an outer product between its randomized context key and its output code.
Therefore, the correlation matrix memory is 
\begin{equation*}
    C = \sum_{j=1}^{M}\phi(s_j)z_{y_j}^{\top}
    \in\R^{p\times d}.
\end{equation*}
We can add a new observation (or removing an existing one) $(s,y)$ simply by adding or subtracting one term $\phi(s)z_y^\top$.
The CMM contains $p d = h d_h d$ scalar parameters.
When $d_h=d/h$, one has $p=d$ and the memory is square (the definition above also covers the rectangular memories).

\paragraph{Retrieval and decoding.}
Given a query context $q \in \gV^h$, the retrieved output vector is $r(q)=C^{\top}\phi(q)\in\R^d$.
The score assigned to candidate token $c$ is obtained by decoding against its output code, i.e.,
\begin{equation}\label{eq:expanded-memo-score}
\widetilde S_c(q)
= \ip{z_c}{r(q)}
= \sum_{j=1}^{M} \ip{\phi(q)}{\phi(s_j)} \ip{z_c}{z_{y_j}}.
\end{equation}
Finally, single-layer \MeMo predicts $\widehat c(q) \in \underset{c\in\gV}{\arg\max}\,\widetilde S_c(q)$.

\Cref{eq:expanded-memo-score} displays two distinct sketches.
The factor $\ip{\phi(q)}{\phi(s_j)}$ approximates positional Hamming similarity in the sequence space, while $\ip{z_c}{z_{y_j}}$ approximates the indicator that the stored label is $c$.
If both families of vectors were exactly orthonormal, then $\widetilde S_c(q) = S_c(q)$ and \MeMo would solve \cref{eq:ideal-prediction} exactly.
The next lemma proves that $\widetilde S_c(q)$ is an unbiased estimator of $S_c(q)$.

\begin{lemma}
\label{prop:unbiased-score}
Suppose that the input codebook $U$, output codebook $Z$, and projection $W$ are mutually independent.
Then, for every fixed dataset $\gD$, query $q$, and candidate token $c$, $\E\widetilde S_c(q)=S_c(q)$.
\end{lemma}

\begin{proof}
The independence of $Z$ from $(U,W)$ and linearity of expectation give 
\[
\E\widetilde S_c(q)
= \sum_{j=1}^{M} \E\ip{\phi(q)}{\phi(s_j)} \E\ip{z_c}{z_{y_j}}
= \sum_{j=1}^{M} K_H(q,s_j)\ind\{c=y_j\}
= S_c(q),
\]
where we used \Cref{prop:expected-key-kernel} and the identity
$\E\ip{z_c}{z_y}=\ind\{c=y\}$.
\end{proof}

\begin{remark}[Tied input and output representations]
The tied construction $Z=U$ uses fewer random objects and corresponds to using the same token representation on both sides of the memory.
In this case, \cref{eq:expanded-memo-score} remains exact, but the two inner products in each summand are generally dependent.
Hence \Cref{prop:unbiased-score} cannot be invoked directly: repeated occurrences of output tokens inside contexts can produce bias terms.
Separating this bias from random cross-talk is an essential step in the analysis of the tied architecture.
\end{remark}

The difference $\widetilde{S}_c(q)-S_c(q)$ is the retrieval error induced by the random sketch.  
Correct next-token prediction depends not only on its absolute size, but on the \emph{margin} of the ideal classifier, i.e.,
\[
    \Delta(q) = S_{c^*(q)}(q) - \max_{c \neq c^*(q)}S_c(q).
\]
A subsequent capacity analysis can therefore be organized in two steps: first obtain simultaneous bounds on the score errors, and then compare those bounds with $\Delta(q)$.
This separates the intrinsic ambiguity of the ideal classification problem from the approximation error and cross-talk introduced by the randomized memory representation.

\section{Analysis}
\label{sec:error-margin-streamlined}

The model approximates the ideal score in two successive stages: the retrieval and the decoding. 
For every label $c$, we define the score after the context sketch but before output decoding,
\begin{equation}
    G_c(q)
    =\sum_{j:y_j=c}\ip{\phi(q)}{\phi(s_j)}.
    \label{eq:predecoding-score-streamlined}
\end{equation}
We can think of $G_c(q)$ as the overall contribution of class $c$ with respect to query $q$.
The passage from $S_c(q)$ to $G_c(q)$ contains only the approximation of the Hamming kernel.
Grouping \cref{eq:expanded-memo-score} by the stored label then gives
\begin{equation}
    \label{eq:decoded-grouped-score-streamlined}
    \widetilde S_c(q)
    = \sum_{y \in \gV} G_y(q) \ip{z_c}{z_y}.
\end{equation}
Hence, the \emph{total score error} can be decomposed in the following two elementary
\begin{equation}
    \widetilde S_c(q)-S_c(q)
    = \bigl(G_c(q)-S_c(q)\bigr) + \bigl(\widetilde S_c(q)-G_c(q)\bigr).
    \label{eq:two-stage-error-streamlined}
\end{equation}
We will refer to the first difference as \emph{context-sketch error}, and to the second one as \emph{output-code interference}.
Intuitively $|G_c(q)-S_c(q)|$ measures how accurately are contexts represented by the random map $\phi$, while $|S_c(q)-G_c(q)|$ indicates how much interference is introduced during decoding phase.

Thus the analysis will be organized in four steps:
\begin{enumerate}[leftmargin=*,itemsep=2pt,topsep=3pt]
    \item In the remainder of this section we identify how much score error can be tolerated without changing the predicted token.
    
    \item In \Cref{sec:context-sketch-streamlined} we give a probabilistic bound to $|G_c(q) - S_c(q)|$. This is needed because the input codebook $U$ and the projection $W$ distort the ideal Hamming similarities.
    
    \item In \Cref{sec:output-interference-streamlined} we give a probabilistic bound to $| \widetilde{S}_c(q) - G_c(q)|$. This accounts for the fact that finite-dimensional output codes are not exactly orthogonal.

    \item Finally, in \Cref{sec:full-guarantee-streamlined} we combine the previous bounds in a main theorem, which gives a sufficient condition for \MeMo and the ideal classifier to select the same token.
\end{enumerate}


We now proceeds with the first point of the analysis. 
Without loss of generality, assume that the ideal classifier has a unique winner $c^* = c^*(q)$. 
For every competitor $c \neq c^*$, we define
\begin{equation*}
    \Delta_c(q)=S_{c^*}(q)-S_c(q),
    \qquad
    \Delta(q)=\min_{c\neq c^*}\Delta_c(q) = S_{c^*}(q) - \max_{c \neq c^*} S_c(q).
\end{equation*}
Thus $\Delta(q)$ is the gap between the largest and second-largest ideal scores.

\begin{theorem}
\label{thm:margin-preservation-streamlined}
Let $\varepsilon_{\mathrm{score}}(q) = \max_{c\in\gV}|\widetilde S_c(q)-S_c(q)|$.
If $\varepsilon_{\mathrm{score}}(q)< \Delta(q)/2$,
then single-layer \MeMo and the ideal Hamming-kernel classifier predict the same token, i.e., $\widehat c(q)=c^*(q)$.
\end{theorem}

\begin{proof}
Fix a competitor $c \neq c^*$.
By definition, we have $\widetilde{S}_{c^*}(q) \geq S_{c^*}(q)-\varepsilon_{\mathrm{score}}(q)$, and $\widetilde{S}_c(q) \leq S_c(q)+\varepsilon_{\mathrm{score}}(q)$.
Subtracting the second inequality from the first yields
\begin{align*}
    \widetilde{S}_{c^*}(q)-\widetilde S_c(q)
    &\geq S_{c^*}(q) - S_c(q) - 2\varepsilon_{\mathrm{score}}(q)\\
    &= \Delta_c(q) - 2\varepsilon_{\mathrm{score}}(q) 
    \geq \Delta(q) - 2\varepsilon_{\mathrm{score}}(q) > 0.
\end{align*}
Thus the randomized score of $c^*$ is strictly larger than the score of every competitor, and the two classifiers have the same argmax.
\end{proof}

The theorem separates two causes of prediction failure.
A small $\Delta(q)$ is an ambiguity of the ideal Hamming-kernel problem, while a large score error is a limitation of the randomized representation. 
By \cref{eq:two-stage-error-streamlined}, it remains to control the context sketch and the output decoding separately.

\subsection{Context-Sketch Approximation}
\label{sec:context-sketch-streamlined}

We now control the first difference in \cref{eq:two-stage-error-streamlined}. 
We first state the result for an arbitrary per-position dimension $d_h$, and then we specialize it to the standard
square-memory choice $d_h=d/h$. 
The argument uses only one idea: the input codebook $U$ and the projection $W$ must jointly preserve every inner product between two token codes.

For each label $c$, let $M_c=|\{j:y_j=c\}|$ be the number of observation in $\gD$ that has $c$ as next token, and $M_{\max}=\max_c M_c$.
Observe that $\sum_{c}M_c = M$, and $M_{\max} \leq M$.

\begin{theorem}[Context-sketch approximation]
\label{thm:context-sketch}
For $\delta \in (0,1)$, we define $L_\delta=\ln\left(\frac{2n}{\sqrt{\delta}}\right)$.
There exist universal constants $c_0, C_0 >0$ such that, if $d\geq c_0L_\delta$, then with probability at least $1-\delta$ over $(U,W)$, simultaneously for every query $q\in\gV^h$,
\begin{equation*}
    \max_{c\in\gV}|G_c(q)-S_c(q)|
    \leq
    \mathcal B_{\mathrm{ctx}}(\delta)
    := C_0 M_{\max}
    \left(
        \sqrt{\frac{L_\delta}{d}}+\frac{L_\delta}{d}
        +\sqrt{\frac{L_\delta}{d_h}}+\frac{L_\delta}{d_h}
    \right).
\end{equation*}
\end{theorem}

\begin{proof}
Set the matrix $A=\sqrt{h}W$, so that $A$ is a standard Gaussian projection into $\R^{d_h}$.
We define the largest error on a token pair by $\eta = \max_{x,y\in\gV} \vert \ip{Au_x}{Au_y}-\ind\{x=y\} \vert$.
Let $\eta_U = \max_{x,y \in \gV} \vert \ip{u_x}{u_y} - \ind\{x=y\} \vert$, and $\eta_A = \max_{x,y \in \gV} \vert \ip{Au_x}{Au_y} - \ip{u_x}{u_y} \vert$.
By triangular inequality $\eta \leq \eta_U + \eta_A$.
For $x = y$, the inner product $\ip{u_x}{u_x} = \norm{u_x}^2$ follows the distribution $d^{-1} \chi^2$.
We apply the standard chi-square concentration bound with parameter $2L_\delta$ , obtaining
\[
\Prob\left( \bigl\vert \norm{u_x}^2 - 1 \bigr\vert > 2\sqrt{\frac{2L_\delta}{d}} + \frac{4L_\delta}{d} \right) = \Prob\left( \bigl\vert d\norm{u_x}^2 - d \bigr\vert > 2\sqrt{2dL_\delta} + 4L_\delta \right) \leq 2e^{-2L_\delta}.
\]
For $x \neq y$, the inner product $\ip{u_x}{u_y}$ is a sum of independent centered sub-exponential random variables.
By Bernstein's inequality \cite[Corollary 2.9.2]{vershynin2018high} we have
\[
\Prob\left( \bigl\vert \ip{u_x}{u_y}\bigr\vert > C_1 \left[\sqrt{\frac{2L_\delta}{d}} + \frac{2L_\delta}{d} \right] \right) \leq 2e^{-2L_\delta},
\]
for some absolute constant $C_1 > 0$.
By union bound over all pairs of tokens $x,y \in \gV$ we have that $\eta_U \leq C_2\left[\sqrt{\frac{L_\delta}{d}} + \frac{L_\delta}{d} \right]$ with probability at least $1 - 2n^2e^{-2L_\delta} = 1-\delta/2$, for some absolute constant $C_2>0$.

We now study the additional error introduced by $A$.
By thinking of the rows of $A$ as random vectors $g_k^\top / \sqrt{d_h} \sim \mathcal{N}(0,I_d/d_h)$, we can rewrite
\[
\ip{Au_x}{Au_y} - \ip{u_x}{u_y} = \frac{1}{d_h}\sum_{k=1}^{d_h}     \left((g_k^\top u_x)(g_k^\top u_y) - \ip{u_x}{u_y}\right),\; \forall x,y \in \gV.
\]
We now fix the codebook $U$.
In this way, the vectors $u_x$ and $u_y$ are deterministic, while the randomness comes only from the independent Gaussian vectors $g_k$.
The terms in the preceding sum are
independent, centered, and sub-exponential.
Applying the Bernstein's inequality again gives
\[
\underset{A}{\Prob}\left( \bigl\vert \ip{Au_x}{Au_y} - \ip{u_x}{u_y} \bigr\vert > C_3 \norm{u_x}\norm{u_y}\left[\sqrt{\frac{2L_\delta}{d_h}} + \frac{2L_\delta}{d_h} \right] \; \middle| \; U\right) \leq 2e^{-2L_\delta}.
\]
For some absolute constant $C_3 > 0$.
Now observe that $\norm{u_x}\norm{u_y} \leq \max_{z} \norm{u_z}^2 \leq 1+ \eta_U$, where the last inequality follows by taking $x = y = z$ in the definition of $\eta_U$.
Therefore, a union
bound over all $n^2$ pairs gives $\eta_A \leq C_4 (1+\eta_U)\left[\sqrt{\frac{L_\delta}{d_h}} + \frac{L_\delta}{d_h}\right]$, with conditional probability at least $1 - \delta/2$, for some absolute constant $C_4 > 0$.

On the event where the preceding bound for $\eta_U$ holds, the assumption $d \geq c_0L_\delta$ gives $\eta_U \leq C_2\left(\frac{1}{\sqrt{c_0}}+\frac{1}{c_0}\right)$.
By choosing the universal constant $c_0$ sufficiently large, we can ensure that $\eta_U\leq1$.
Consequently, $1+\eta_U \leq 2$, and hence $\eta_A \leq 2C_4 \left[ \sqrt{\frac{L_\delta}{d_h}} +\frac{L_\delta}{d_h} \right]$.

The bound for $\eta_U$ fails with probability at most $\delta/2$, while for every codebook $U$ satisfying that bound, the conditional estimate for $\eta_A$ fails with probability at most $\delta/2$.
Therefore, both estimates hold simultaneously with probability at least $1-\delta$.
On this event,
\begin{align*}
\eta
&\leq\eta_U+\eta_A
\leq C_2\left[ \sqrt{\frac{L_\delta}{d}} + \frac{L_\delta}{d} \right] + 2C_4\left[ \sqrt{\frac{L_\delta}{d_h}} + \frac{L_\delta}{d_h} \right]
\leq C_0\left[ \sqrt{\frac{L_\delta}{d}} + \frac{L_\delta}{d} + \sqrt{\frac{L_\delta}{d_h}} + \frac{L_\delta}{d_h} \right],
\end{align*}
where $C_0=\max\{C_2, 2C_4\}$.
Finally, for every query $q \in \gV^h$ and every class $c \in \gV$,
\begin{align*}
|G_c(q)-S_c(q)|
&\leq \frac{1}{h} \sum_{j:y_j=c}\sum_{r=1}^h \left| \langle Au_{q_r},Au_{(s_j)_r}\rangle -\ind\{q_r=(s_j)_r\} \right|\\
&\leq \frac{1}{h}\sum_{j:y_j=c}\sum_{r=1}^h\eta
= M_c\eta \leq M_{\max}\eta. \qedhere
\end{align*}
\end{proof}

\paragraph{Square-memory specialization.}
For the standard \MeMo choice $d_h=d/h$ (i.e., with a squared memory matrix), we have $d_h\leq d$, so the terms that depends on $d_h$ dominate, therefore
$\mathcal{B}_{\mathrm{ctx}}(\delta) = M_{\max} O\Big( \sqrt{\frac{hL_\delta}{d}}+\frac{hL_\delta}{d} \Big)$.
When $d\gg h\log{(n/\delta)}$, the square-root term dominates, and $d\gtrsim h\log{(n/\delta)}/\varepsilon^2$ suffices to make this bound at most $\varepsilon M_{\max}$.
When $d \lesssim h\log{(n/\delta)}$, the bound does not guarantee a small relative error, although the actual error may still be small.
In the worst case $M_{\max}=M$.
The linear dependence on $M_{\max}$ is unavoidable for a uniform result, because duplicated contexts repeat the same kernel error coherently.
However, the proof of \Cref{thm:context-sketch} is pessimistic because it bounds every token-position comparison by the same worst-case error $\eta$.
In practice, different token-position pairs may produce errors of different magnitudes and signs.
A more refined, data-dependent analysis could exploit their repetition frequencies and possible cancellations, potentially leading to tighter bounds.

\subsection{Output-Code Interference}
\label{sec:output-interference-streamlined}

We now analyze the second difference in
\cref{eq:two-stage-error-streamlined}.
Fix the input codebook $U$, the projection $W$, the dataset, and the query.
Under this conditioning, the values $G_y(q)$ are deterministic and the only remaining randomness is in the independent Gaussian output codes $Z$.
From \cref{eq:decoded-grouped-score-streamlined}
\begin{equation}
    \widetilde S_c(q)-G_c(q)
    =G_c(q)\bigl(\norm{z_c}^2-1\bigr)
     +\sum_{y\neq c}G_y(q)\ip{z_c}{z_y}.
    \label{eq:self-and-cross-output-error-streamlined}
\end{equation}
The first term is the \emph{norm error} of the candidate code $z_c$, while the second is the \emph{cross-talk} from the competing output codes.

We now introduce the quantity $\widehat{\mathcal C}_H(q)=\sum_{y\in\mathcal V}G_y(q)^2$,
that is the squared norm of the score vector before the decoding phase in \cref{eq:decoded-grouped-score-streamlined}.
It measures the total energy of the evidence assigned by the query to the output labels and therefore controls the amount of interference caused by their non-orthogonal codes.
In particular, the typical cross-talk scale is of order $\sqrt{\widehat{\mathcal C}_H(q)/d}$.
Observe that this is a quantity determined by the data and the query.


\begin{theorem}[Output-code interference]
\label{thm:output-code-interference}
Let $\delta\in(0,1)$ and set $t = \ln\left(\frac{4n}{\delta}\right)$, $\alpha = 2\sqrt{\frac{t}{d}}+\frac{2t}{d}$. 
Conditionally on $(U,W)$, with probability at least $1-\delta$ over the output codebook $Z$, every candidate $c \in \gV$ satisfies
\begin{equation*}\label{eq:output-interference-bound-streamlined}
\vert \widetilde{S}_c(q) - G_c(q) \vert \leq \mathcal{B}_{\mathrm{out}}(q,\delta) := \sqrt{\widehat{\mathcal C}_H(q)}
\left( \alpha + \sqrt{\frac{2t(1+\alpha)}{d}} \right).
\end{equation*}
\end{theorem}
%
%
\begin{proof}
Fix an element $c \in \gV$ and define $L_c(q)^2=\sum_{y\neq c}G_y(q)^2 = \widehat{\mathcal{C}}_H(q) - G_c(q)^2$, and the vector $v_c(q)=\sum_{y\neq c}G_y(q)z_y$.
Since the vectors $z_y$ (with $y\neq c$) are independent and distributed as $\mathcal N(0,I_d/d)$, we have that $v_c(q)\sim \mathcal{N}\left(0,\frac{L_c(q)^2}{d}I_d\right)$, and $v_c(q)$ is independent of $z_c$.
Conditional on $z_c$, the cross-talk term in \cref{eq:self-and-cross-output-error-streamlined} is therefore one-dimensional Gaussian distributed as
\begin{equation}
\ip{z_c}{v_c(q)}\mid z_c
    \sim
    \mathcal{N}\left(
        0,\frac{L_c(q)^2}{d}\norm{z_c}^2
    \right).
    \label{eq:conditional-cross-talk-distribution-streamlined}
\end{equation}

Because $\norm{z_c}^2 \sim d^{-1}\chi_d^2$, standard chi-square concentration \citep{vershynin2018high} gives
\[
\Prob\bigl( \bigl\vert \norm{z_c}^2-1 \bigr\vert > \alpha \bigr)\leq 2e^{-t}.
\]
On the complementary event, $\norm{z_c}\leq\sqrt{1+\alpha}$.
The Gaussian tail bound applied to \cref{eq:conditional-cross-talk-distribution-streamlined} gives, conditionally on $z_c$, 
\[\Prob\biggl( \vert \ip{z_c}{v_c(q)} \vert > L_c(q)\norm{z_c}\sqrt{\frac{2t}{d}} \vert z_c \biggr)\leq 2e^{-t}.
\]
Combining the two events with \cref{eq:self-and-cross-output-error-streamlined}, with probability at least $1-4e^{-t}$ we have
\[
\vert \widetilde{S}_c(q)-G_c(q) \vert
\leq \vert G_c(q) \vert \alpha + L_c(q) \sqrt{1+\alpha}\sqrt{\frac{2t}{d}}
\leq \sqrt{\widehat{\mathcal C}_H(q)} \left( \alpha + \sqrt{\frac{2t(1+\alpha)}{d}} \right),
\]
where last inequality holds since both $\vert G_c(q) \vert$ and $L_c(q)$ are at most $\sqrt{\widehat{\mathcal C}_H(q)}$. Finally, since $4e^{-t}=\delta/n$, a union bound over the $n$ candidates proves simultaneous claim.
\end{proof}

\subsection{Full Prediction Guarantee}
\label{sec:full-guarantee-streamlined}

The two concentration results from \Cref{sec:context-sketch-streamlined,sec:output-interference-streamlined} can now be combined without assuming independence between individual stored examples.

\begin{theorem}[Exact recovery of the ideal prediction]
\label{thm:full-recovery-streamlined}
Assume the untied model, with arbitrary positive integer $d_h$.
Fix a dataset and query, suppose that the ideal winner $c^*(q)$ is unique, and let $\delta\in(0,1)$.
If $d\geq c_0L_{\delta/2}$ (as defined in \Cref{thm:context-sketch}), then, with probability at least $1-\delta$ over $(U,W,Z)$,
\begin{equation}\label{eq:complete-score-bound-streamlined}
\varepsilon_{\mathrm{score}}(q) \leq
\mathcal{B}_{\mathrm{ctx}}\left(\frac{\delta}{2}\right)
+\mathcal{B}_{\mathrm{out}}\left(q,\frac{\delta}{2}\right).
\end{equation}
On the same event, if $\mathcal{B}_{\mathrm{ctx}}\left(\frac{\delta}{2}\right)
+ \mathcal{B}_{\mathrm{out}}\left(q,\frac{\delta}{2}\right) < \frac{\Delta(q)}{2}$, then $\widehat c(q)=c^*(q)$.
\end{theorem}

\begin{proof}
Apply \Cref{thm:context-sketch} with failure probability $\delta/2$.
On the resulting event, $\max_c \vert G_c(q)-S_c(q) \vert \leq \mathcal{B}_{\mathrm{ctx}}\left(\frac{\delta}{2}\right)$.
For every fixed realization of $(U,W)$, \Cref{thm:output-code-interference}, again with failure probability $\delta/2$, gives $\max_c \vert \widetilde{S}_c(q) - G_c(q) \vert  \leq \mathcal{B}_{\mathrm{out}}\left(q,\frac{\delta}{2}\right)$.
A union bound and \cref{eq:two-stage-error-streamlined} prove
\cref{eq:complete-score-bound-streamlined}.
The prediction statement follows from \Cref{thm:margin-preservation-streamlined}.
\end{proof}

\Cref{thm:full-recovery-streamlined} is the main recovery theorem.
It separates the context error, controlled by $(d,d_h,M_{\max})$, from output interference, controlled by $d$ and the distribution of the class scores.

If the context sketch were exact, then $G_y(q) = S_y(q)$ and $\widehat{\mathcal{C}}_H(q) = \mathcal{C}_H(q) = \sum_{y\in\gV}S_y(q)^2$.
Intuitively, suppose that $M$ stored contexts all have Hamming similarity $\rho$ with the query.
If their labels are distinct, then $\mathcal C_H(q)=M\rho^2$, while if they all have the same label $\mathcal C_H(q)=M^2\rho^2$.
In general, if $k$ stored contexts share the same label $y$, their contributions are first aggregated into $S_y(q)=k\rho$. For any competing candidate $c\neq y$, they therefore produce the single cross-talk term $k\rho\langle z_c,z_y\rangle$, whose variance scales as $k^2\rho^2/d$, since $\operatorname{Var}(\langle z_c,z_y\rangle)=1/d$. This quadratic dependence on $k$ reflects the coherent accumulation of repeated labels.
Consequently, output interference depends on how the total score is distributed across labels, as captured by $\mathcal C_H(q)=\sum_y S_y(q)^2$, rather than on memory size alone.



\paragraph{A square-memory recovery condition.}
The output bound contains the realized collision energy
$\widehat{\mathcal C}_H(q)$.
It can be replaced by ideal, count-only quantities.
For readability, we give this explicit specialization under $d_h=d/h$.
Let $\ell_{\gD} = \bigl\vert \{y_j:1\leq j\leq M\} \bigr\vert$ be the number of distinct labels that actually occur in the memory\footnote{Observe that $\ell_\gD \leq n$}.

\begin{corollary}
\label{cor:deterministic-recovery-streamlined}
Assume the untied square-memory model, so that $h$ divides $d$ and $d_h=d/h$, and suppose that the ideal winner is unique.
Fix $\delta \in (0,1)$, and set $L = \ln\left(\frac{16n}{\delta}\right)$, $\alpha = 2\sqrt{\frac{L}{d}} + \frac{2L}{d}$, and $\gamma = \sqrt{2L(1+\alpha)/d}$.
If $d \geq c_0L$, we define
\[
\beta_{\mathrm{ctx}} = C_0 M_{\max}\biggl(\sqrt{\frac{hL}{d}}+\frac{hL}{d}\biggr), \qquad
\beta_{\mathrm{out}} = \left( \sqrt{\mathcal C_H(q)} + \sqrt{\ell_{\gD}} \beta_{\mathrm{ctx}}\right)(\alpha+\gamma),
\]
for some absolute constants $c_0, C_0 > 0$.
Then, with probability at least $1-\delta$, we have that $\varepsilon_{\mathrm{score}}(q) \leq \beta_{\mathrm{ctx}}+\beta_{\mathrm{out}}$.
Consequently, if the $\beta_{\mathrm{ctx}}+\beta_{\mathrm{out}} < \Delta(q)/2$, then $\widehat c(q)=c^*(q)$.
\end{corollary}

\begin{proof}
Apply \Cref{thm:context-sketch} with failure probability $\delta/2$, so that the context error is at most $\beta_{\mathrm{ctx}}$.
Only the $\ell_{\gD}$ observed labels have nonzero scores, hence
\begin{equation}
    \sqrt{\widehat{\mathcal C}_H(q)}
    =\norm{(G_y(q))_{y\in\gV}}_2
    \leq\sqrt{\mathcal C_H(q)}
      +\sqrt{\ell_{\gD}}\,\beta_{\mathrm{ctx}}.
    \label{eq:simple-collision-envelope-streamlined}
\end{equation}
We now apply \Cref{thm:output-code-interference} conditionally on $(U,W)$, with failure probability $\delta/2$.
Its tail parameter is at most $L$, so \cref{eq:simple-collision-envelope-streamlined} bounds the output error by $\beta_{\mathrm{out}}$.
Finally, a union bound and \cref{eq:two-stage-error-streamlined} prove the score bound.
The prediction claim follows from \Cref{thm:margin-preservation-streamlined}.
\end{proof}

\begin{remark}[Readable square-memory asymptotics]
\label{rem:asymptotic-recovery-streamlined}
Since $K_H\in[0,1]$, the ideal collision energy and the number of observed
labels satisfy $\mathcal C_H\leq M M_{\max}$, and $\ell_{\gD}\leq\min\{n,M\}$.
For $d\gg hL$, the two terms $\beta_{\mathrm{ctx}}$ and $\beta_{\mathrm{out}}$ in the bound have the asymptotic scales $\beta_{\mathrm{ctx}} = O\big( M_{\max}\sqrt{hL/d}\big)$, and $\beta_{\mathrm{out}} = O\big( \sqrt{M M_{\max}}\sqrt{L/d} + \sqrt{\ell_{\gD}}\,M_{\max}\sqrt{h}L/d \big)$.
The second term in $\beta_{\mathrm{out}}$ is the interaction between context and output distortions.
If $d \gg \ell_{\gD}L$, it is of lower order in the total error, and then the error simplifies as $\varepsilon_{\mathrm{score}}(q) = O\big( (M_{\max}\sqrt h+\sqrt{M M_{\max}}) \sqrt{L/d} \big)$.
In the worst case $M_{\max}=M$, this reduces to $\varepsilon_{\mathrm{score}}(q) = O\bigl(M\sqrt{hL/d}\bigr)$.
Ignoring constants, the corresponding 
recovery requirement is $d \gtrsim \frac{L(hM_{\max}^2+M M_{\max})}{\Delta(q)^2}$.

In this bound, the vocabulary size $n$ appears only inside a logarithm.
For a square memory, i.e., $d_h=d/h$, so $d$ grows linearly with $h$ if $d_h$ is kept fixed.
The term $\sqrt{M M_{\max}}$ reflects the fact that contexts with the same label add their contributions before output decoding.
Finally, the recovery condition is easier to satisfy when the ideal margin $\Delta(q)$ is larger, because the prediction is more robust to \MeMo's approximation errors.
\end{remark}

\section{Experimental Evaluation}

\subsection{Synthetic Validation}
\label{sec:synthetic-validation}
In this section we perform a first synthetic validation of the two-stage analysis of \Cref{sec:error-margin-streamlined}.
More precisely, we use a controlled synthetic experiment to test separately the context-sketch approximation, output-code interference, and end-to-end recovery.


\paragraph{Dataset construction.}
We set $n=1024$, $M=4096=4n$, and $h\in\{4,8,16\}$.
For each $h$, we construct a separate dataset with fixed query $q=(1,\ldots,h)$.
The first $64$ labels are active in $\gD$ (i.e., appear has next-token) and each has $64$ contexts, while the remaining $960$
labels have ideal score zero but remain possible output candidates.
For the distinguished label $c^*=c^*(q)=1$, exactly $40h$ of the $64h$ token-position entries are selected uniformly without replacement and matched to the corresponding query token.
For each other active label $c \neq c^*$, exactly $24h$ entries are matched.
Every remaining entry at position $r$ is drawn uniformly from $\gV\setminus\{q_r\}$, preventing accidental matches.
Hence, we have $S_{c^*}(q)=40$, and $S_c(q)=24$ for every other active label $c \neq c^*$, and the ideal margin is $\Delta(q)=40-24=16$. 
Scaling the planted match counts with $h$ keeps the ideal classification problem unchanged across the three context lengths.

\paragraph{Merged-label instance.}
To increase label repetition without changing the ideal margin, we construct a second labeling of the $h=8$ dataset.
The contexts remain unchanged, but the $64$ active labels are mapped into $16$ groups according to the map $c \mapsto 1+\bigl((c-1)\bmod 16\bigr)$.
Each merged label therefore contains $256$ contexts.
The group containing $c^*$ has score $40+3 \cdot 24 = 112$, whereas every other merged label has score $4\cdot24=96$.
Thus the margin $\Delta(q)$ remains $16$, while $M_{\max}$ increases from $64$ to $256$.
In this case the collision energy increases from $\mathcal C_H(q)=40^2+63 \cdot 24^2 = 37,888$ to $\mathcal C_H^{\mathrm{merged}}(q) = 112^2 + 15 \cdot 96^2 = 150,784$.
The two label assignments consequently have the same ideal winner and margin,
but different class loads and levels of output interference.

\paragraph{Experimental protocol.}
We use the square-memory choice $d_h = d/h$ and evaluate $d \in \{128, 256, 512, 1024, 2048, 4096\}$.
For each configuration, the dataset and query are fixed and $160$ independent Gaussian realizations are sampled.

The left panel of \Cref{fig:synthetic-validation} uses the original
$64$-label datasets and reports the quantity $\frac{\max_c|G_c(q)-S_c(q)|}{M_{\max}}$ for $h \in \{4,8,16\}$, thereby isolating the context-sketch error.
The center panel reports $\max_c|\widetilde S_c(q)-S_c(q)|$ for the original and merged assignments, thereby isolating output-decode interference.\footnote{Here we set $G(q)=S(q)$, so the decoder receives the exact ideal scores instead of the approximate context scores. This removes the context-sketch error and allows us to measure \emph{only} the output-code interference.
}
The right panel uses the complete untied
model and reports the fraction of trials satisfying
$\widehat{c}(q)=c^*$.
It includes the original datasets for all
three values of $h$ and the merged instance for $h=8$.
This is agreement with the ideal classifier, not the accuracy against an external ground truth.


\begin{figure}[htbp]
    \centering
    \includegraphics[width=.9\textwidth]{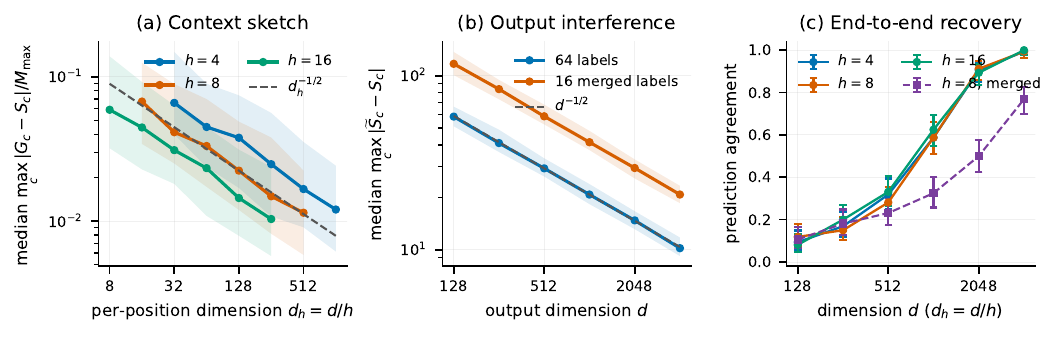}
    \caption{%
    Controlled synthetic validation with $n=1024$, $M=4096$, and $160$ independent realizations per point.
    Curves in the left and center panels show medians, with $10$th-$90$th percentile bands, while error bars on the right are pointwise Wilson $95\%$ intervals.
    \textbf{Left:} normalized context-sketch error for the original datasets with $h\in\{4,8,16\}$.
    \textbf{Center:} output-decoding error from the exact ideal scores for the original and merged assignments.
    \textbf{Right:} agreement of the complete untied model with the ideal classifier, including the merged $h=8$ instance.}
    \label{fig:synthetic-validation}
\end{figure}

\paragraph{Results.}
The left panel is consistent with the predicted $d_h^{-1/2}$ decay in the well-resolved regime.
In the center panel, merging the labels increases $\mathcal C_H(q)$ by approximately a factor of four and the typical output error by approximately a factor of two, consistently with the predicted $\sqrt{\mathcal C_H(q)/d}$ scale.
End-to-end agreement increases with $d$ and tends to one for the original assignment, while the merged instance recovers more slowly, despite exhibiting the same ideal margin.
These observations support the qualitative scalings predicted by the analysis, but do not constitute a real-world data benchmark.
This assessment is addressed in \Cref{sec:real-data-evaluation}.

\subsection{Real-Data Evaluation}
\label{sec:real-data-evaluation}

In this section we ask whether the same picture remains useful on real sequential data.
The experiment keeps three questions separate: (i) prediction of the true next token, (ii) agreement with the exact Hamming-kernel classifier, and (iii) computational cost.
The second quantity measures the quality of the random approximation, although it is not the same as accuracy against the observed next token.

\paragraph{Synthetic Dataset Construction.}
We use the standard \texttt{WikiText-2} training, validation, and test splits
\citep{merity2016pointer}.
Whitespace tokenization with an explicit \texttt{<eos>} marker gives a training vocabulary of $n = 33,278$ input tokens.
The task is restricted to the $64$ most frequent training targets\footnote{After excluding \texttt{<unk>}.}.
These targets cover $45.12\%$ of eligible training positions and $44.79\%$ of eligible test positions.
This is therefore a controlled next-token experiment with $64$-classes, not a full language-model benchmark.

We sample $100,000$ training, $2,000$ validation, and $5,000$ test positions with a fixed seed.
Every sampled position has at least $16$ preceding tokens on the same line.
The same positions considered are consequently used for $h \in \{4,8,16\}$, and no context crosses a line boundary.
All methods receive exactly the same examples. 
Finally, \MeMo uses the untied square construction $d_h=d/h$, dimensions $d \in \{128, 256, 512, 1024\}$, and five independent codebook seeds.

\paragraph{Compared methods.}
The \emph{sum prototype} is the exact classifier in \cref{eq:count-table-score}, implemented with token-position-label count tables rather than by linear scanning the training set.
The \emph{mean prototype} divides each score by its class size and therefore removes the empirical class frequency factor.
We also use a count-based suffix model that backs off from at most five preceding tokens, a linear SVM trained with stochastic hinge loss on the positional one-hot features $\Psi$, and exact Hamming $k$-nearest neighbors.
For the latter, $k \in \{1, 5, 25\}$ is chosen on validation data ($k=25$ is selected for all three context lengths).
Ties are resolved deterministically.
We report top-$1$ and top-$5$ accuracy, macro-F1, and \emph{agreement}\footnote{That is, the percentage of queries for which a method predicts the same token as the exact Hamming classifier.} with the sum-prototype winner.
Since the ideal winner is unique on at least $99.34\%$ of the test queries, agreement is reported on that unambiguous subset.
We emphasize that the agreement with the \emph{exact sum prototype is our primary measure} of how faithfully \MeMo implements the positional Hamming-kernel classifier, independently of that classifier's accuracy against the observed next token.

\begin{table}[htbp]
    \centering
    \small
    \caption{\texttt{WikiText-2} results for $h=8$.
    All entries, except agreement, are measured against the observed next token.
    Agreement is with the exact sum-prototype prediction on queries having a unique ideal winner.
    \MeMo values are mean $\pm$ standard deviation over five codebook seeds.}
    \label{tab:real-data-results}
    \begin{tabular}{lrrr|r}
        \toprule
        Method & Top-1 & Top-5 & Macro-F1 & Agreement \\
        \midrule
        Sum prototype             & $22.30$ & $53.44$ & $3.36$  & $100.00$ \\
        Mean prototype            & $9.84$  & $29.58$ & $6.68$  & $9.87$ \\
        Suffix backoff            & $31.96$ & $63.32$ & $19.15$ & $37.60$ \\
        Linear SVM                & $34.28$ & $60.96$ & $18.80$ & $38.06$ \\
        Hamming $25$-NN           & $25.66$ & $55.62$ & $9.01$  & $53.54$ \\
        \midrule
        \MeMo, $d=128$             & $16.35\pm0.79$ & $40.48$ & $2.63$ & $47.22\pm6.07$ \\
        \MeMo, $d=256$             & $18.76\pm1.27$ & $46.35$ & $2.94$ & $54.70\pm3.16$ \\
        \MeMo, $d=512$             & $20.18\pm0.40$ & $48.66$ & $3.23$ & $63.46\pm1.89$ \\
        \MeMo, $d=1024$            & $21.50\pm0.76$ & $50.99$ & $3.30$ & $72.79\pm4.55$ \\
        \bottomrule
    \end{tabular}
\end{table}

\begin{figure}[htbp]
    \centering
    \includegraphics[width=.9\textwidth]{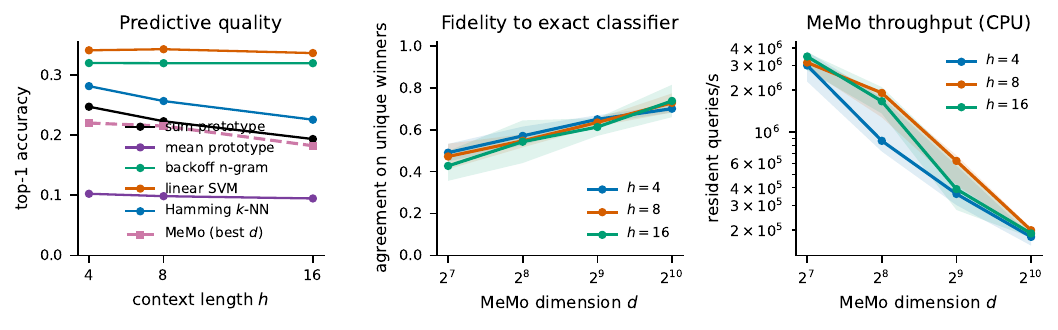}
    \caption{Real-data results on \texttt{WikiText-2} dataset.
    \textbf{Left:} top-$1$ accuracy of the five baselines and the best observed \MeMo dimension for each $h$.
    \textbf{Center:} \MeMo agreement with the exact sum prototype on queries with a unique ideal winner. The curves show means over five codebook seeds, and bands show the standard deviation.
    \textbf{Right:} number of queries per second by \MeMo on an Intel Core i7-13700H CPU, for different values of $d$ and $h$. Queries are processed in batches of $512$ and loaded into memory before timing.}
    \label{fig:real-data-benchmark}
\end{figure}

\paragraph{Results.}
The linear SVM gives the highest top-$1$ accuracy, followed closely by the suffix model.
This is not a contradiction: these baselines learn a discriminative rule or preserve token conjunctions, whereas the ideal single-layer \MeMo classifier is an additive, count-weighted Hamming rule.
The poor top-$1$ accuracy of the mean prototype also shows that class frequency is an important part of the exact sum-prototype prediction on this task.

For \MeMo, increasing $d$ consistently improves fidelity to the exact classifier (the full curves, are reported in the \Cref{fig:real-data-benchmark}).
At $h=8$, mean agreement rises from $47.22\%$ at $d=128$ to $72.79\%$ at $d=1024$.
At the latter dimension, top-$1$ accuracy is $21.50\%$, only $0.80$ percentage points below the exact sum prototype (\Cref{tab:real-data-results}).
The deployed two-stage \MeMo state (the projected input codes, CMM, and output codes) occupies $20.50$ MiB, compared with $65.00$ MiB for the dense exact count table, about $68\%$ less memory (\Cref{tab:gpu-inference}).
At $d=512$, \MeMo uses $9.25$ MiB and obtains $20.18\%$ top-1 accuracy and $63.46\%$ agreement.
Thus stronger compression gives a visible but gradual loss of fidelity rather than an abrupt failure.


\paragraph{GPU inference.}
We start observing that the two matrix products in \MeMo's retrieval rule can be combined before inference.
If $Z$ has the output codes as columns, then we can write $\widetilde{S}(q) = \phi(q)^{\top}(CZ)$.
The matrix $CZ$ is computed once, so each query needs only one matrix product after its features are formed.
We will refer to this implementation as \emph{fused} \MeMo, and gives the same scores as the two-stage retrieval (up to floating-point rounding).
We empathize that it does not remove the output-code interference analyzed above.

We benchmark both versions on an NVIDIA GeForce RTX 4060 Laptop GPU at $h=8$, using the same device for the sum prototype, mean prototype, and linear SVM.
The inputs and model state are already loaded on the GPU.
Each measurement includes the time needed to compute the scores and choose a label on the GPU, but excludes training, data transfers, and Python overhead.
We cycle through batches of test queries.
We run $20$ batches to warm up the GPU, then time $100$ batches. We repeat this three times and report the median result.
The state sizes in \Cref{tab:gpu-inference} count the tensors required for inference, not peak GPU allocation.

\begin{table}[htbp]
    \centering
    \small
    \caption{Inference on the RTX 4060 Laptop GPU at $h=8$.
    Throughput is in millions of queries per second (Mq/s), while $B$ is batch size.
    }
    \label{tab:gpu-inference}
    \begin{tabular}{lrrrr}
        \toprule
        Method & State (MiB) & $B=1$ ($\mu$s) & $B=256$ (Mq/s) & $B=1024$ (Mq/s) \\
        \midrule
        Sum prototype                  & $65.00$ & $18.59$ & $11.36$ & $33.33$ \\
        Mean prototype                 & $65.00$ & $22.66$ & $9.62$  & $31.08$ \\
        Linear SVM                     & $65.00$ & $26.35$ & $8.33$  & $29.41$ \\
        \MeMo fused, $d=128$            & $2.06$  & $20.22$ & $12.56$ & $48.41$ \\
        \MeMo fused, $d=256$            & $4.12$  & $20.21$ & $11.90$ & $40.00$ \\
        \MeMo fused, $d=512$            & $8.25$  & $20.19$ & $11.90$ & $27.78$ \\
        \MeMo fused, $d=1024$           & $16.50$ & $20.16$ & $9.34$  & $16.37$ \\
        \MeMo two-stage, $d=1024$       & $20.50$ & $25.31$ & $2.91$  & $3.21$  \\
        \bottomrule
    \end{tabular}
\end{table}

At batch size $1024$, fused \MeMo with $d=256$ processes $40.00$ million queries per second, compared with $33.33$ million for the exact sum prototype, while using $4.12$ MiB rather than $65.00$ MiB.
Its mean top-$1$ accuracy is $18.76\%$, versus $22.30\%$ for the exact classifier.
With $d=512$, accuracy rises to $20.18\%$ and the state remains $7.9$ times smaller, but throughput falls to $27.78$ million queries per second.
At $d=1024$, \MeMo comes closer to the exact prediction but becomes slower still.
The literal two-stage implementation is substantially slower at this dimension, which shows the practical value of the algebraic fusion when only class scores are needed.
The linear SVM has the highest predictive accuracy ($34.28\%$), so \MeMo does not dominate the alternatives on every metric.

The CPU measurements in \Cref{fig:real-data-benchmark}, and the GPU measurements above, describe different hardware and should not be compared as a speedup.

Overall, \MeMo can offer a favorable trade-off between accuracy, memory, and throughput thanks to a compressed representation and highly parallelizable matrix operations. 
The present experiment illustrates that trade-off, but does not establish universal superiority over either the exact count table or a discriminatively trained classifier.

\bibliography{reference}
\bibliographystyle{iclr2027_conference}

\end{document}

%% file: math_commands.tex
\usepackage{amsmath,amsfonts,bm,amssymb,amsthm}

\def\eqref#1{equation~\ref{#1}}

\def\1{\bm{1}}

\DeclareMathAlphabet{\mathsfit}{\encodingdefault}{\sfdefault}{m}{sl}
\SetMathAlphabet{\mathsfit}{bold}{\encodingdefault}{\sfdefault}{bx}{n}

\def\gD{{\mathcal{D}}}

\def\gV{{\mathcal{V}}}

\newcommand{\E}{\mathbb{E}}

\newcommand{\R}{\mathbb{R}}

\newcommand{\Prob}{\mathbb{P}}
\newcommand{\ind}{\mathbf{1}}
\newcommand{\ip}[2]{\left\langle #1,#2\right\rangle}
\newcommand{\norm}[1]{\left\lVert #1\right\rVert}

\usepackage{xspace}
\newcommand{\MeMo}{\textsc{MeMo}\xspace}

\usepackage[hidelinks]{hyperref}
\usepackage[nameinlink,noabbrev]{cleveref}

\newtheorem{theorem}{Theorem}

\newtheorem{lemma}{Lemma}
\newtheorem{corollary}[theorem]{Corollary}
\theoremstyle{definition}

\newtheorem{problem}[theorem]{Problem}
\theoremstyle{remark}
\newtheorem{remark}[theorem]{Remark}

\crefname{lemma}{lemma}{lemmas}
\Crefname{lemma}{Lemma}{Lemmas}